\documentclass{article}

\usepackage{PRIMEarxiv}

\usepackage[utf8]{inputenc} 
\usepackage[T1]{fontenc}    
\usepackage{times}
\usepackage{latexsym}
\usepackage{amsmath}
\usepackage{amssymb} 
\usepackage{hyperref}
\usepackage{graphicx}
\usepackage{svg}
\usepackage{makecell}
\usepackage{hyphenat}
\usepackage{amsthm}
\newtheorem{theorem}{Theorem}
\usepackage{pgfplots}
\usepackage{url} 
\pgfplotsset{compat=1.17}
\usepackage{float}
\usepackage{booktabs}
\usepackage{enumitem}
\usepackage{amsthm}
\newtheorem{lemma}{Lemma}
\usetikzlibrary{positioning, arrows.meta, shapes, backgrounds}
\usepackage{tikz}
\usetikzlibrary{positioning, shapes, backgrounds, arrows.meta}
\usepackage{subcaption}
\usepackage{pgfplots}
\pgfplotsset{compat=1.17}
\usepgfplotslibrary{groupplots}
\usepackage{xcolor}
\newtheorem{proposition}{Proposition}

\usepackage[T1]{fontenc}

\usepackage[utf8]{inputenc}

\usepackage{microtype}

\usepackage{inconsolata}

\usepackage{graphicx}
\usepackage{booktabs} 

\title{Learning What to Remember: Adaptive Probabilistic Memory Retention for Memory-Efficient Language Models\\ \large A Probabilistic Framework for Memory-Constrained Language Modeling}

\author{S M Rafiuddin \\
  Department of Computer Science \\
  Oklahoma State University \\
  Stillwater, OK, USA \\
  \texttt{srafiud@okstate.edu} \\\And
  Muntaha Nujat Khan \\
  Department of English \\
  Oklahoma State University \\
Stillwater, OK, USA \\
  \texttt{munkhan@okstate.edu} \\}

\begin{document}
\maketitle


\begin{abstract}
Transformer attention scales quadratically with sequence length $O(n^2)$, 
limiting long-context use. We propose \emph{Adaptive Retention}, a probabilistic, layer-wise token 
selection mechanism that learns which representations to keep under a strict global budget $M$. 
Retention is modeled with Bernoulli gates trained via a Hard-Concrete/variational relaxation and 
enforced with a simple top-$M$ rule at inference, making the method differentiable and drop-in for 
standard encoders. Across classification, extractive QA, and long-document summarization, keeping only $30$--$50\%$ of tokens preserves $\geq 95\%$ of full-model performance while cutting peak memory by 
$\sim 35$--$45\%$ and improving throughput by up to $\sim 1.8\times$. This architecture-agnostic 
approach delivers practical long-context efficiency without modifying base attention or task heads.\footnote{Accepted at EMNLP 2025 Findings as a short paper.}

\end{abstract}

\section{Introduction}

Transformer-based language models have achieved remarkable success across a wide range of NLP tasks \cite{vaswani2017attention,devlin2019bert,radford2019language,brown2020language}, but their memory requirements grow quadratically with sequence length, posing significant challenges for long‑context processing \cite{tay2022efficient}. To address these limitations, various approaches have explored sparse attention patterns to reduce computational and memory overhead \cite{beltagy2020longformer,zaheer2020big,child2019generating,dao2023flashattention} and compressed memory representations to extend effective context lengths \cite{rae2020compressive,wu2022memorizing,kimcompressed}. Recent work has further proposed dynamic memory and computation strategies, such as internet‑scale memory compression \cite{zemlyanskiy2024memory}, token merging for efficient inference \cite{bolya2023token}, and streaming memory management \cite{xiao2024efficient}. However, existing methods often either rely on architectural modifications or apply fixed heuristics without explicitly modeling adaptive token retention under strict memory budgets. Inspired by advances in adaptive computation time \cite{graves2016adaptive,xin2020deebert,liu2020fastbert} and improvements in selective attention \cite{leviathan2024selective}, our approach formulates memory retention as a probabilistic learning problem, enabling language models to dynamically retain the most informative token representations while respecting a global memory constraint. Through this formulation, our approach achieves significant memory savings with minimal performance degradation, offering a flexible and efficient solution applicable to standard transformer architectures.

\begin{figure*}[ht]
  \centering
  \includegraphics[width=\textwidth]{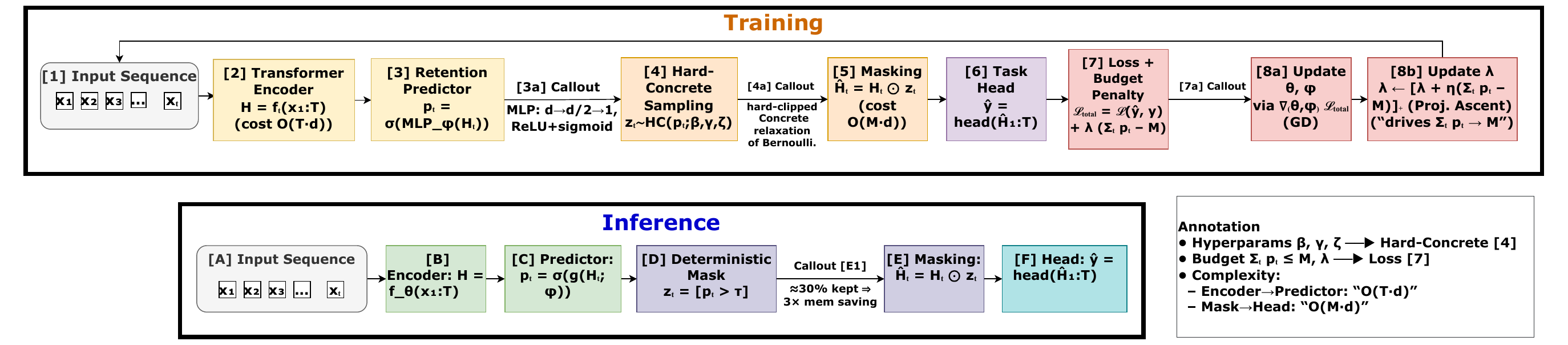}
\caption{\textbf{Adaptive Retention: layer-wise probabilistic token selection.}
At each Transformer block, a lightweight gated scorer produces per-token probabilities trained with a Hard–Concrete relaxation.
At inference, we keep the top-$M_l$ tokens per layer ($M_l=\lfloor \rho T_l \rfloor$), forwarding only those to the next block.
The active sequence length shrinks with depth, yielding cumulative compute and memory savings while leaving base attention unchanged.
Symbols: $H^{l}$ token states; $s^{l}$ scores; $p^{l}$ probabilities; $\rho$ target ratio; $M_l$ retained count.}

  \label{fig:adaptive-retention}
\end{figure*}

\section{Related Work}

\noindent\textbf{Memory‐Efficient Self‐Attention.}  
The quadratic memory cost of Transformer attention \cite{vaswani2017attention} has prompted variants such as Reformer’s locality‐sensitive hashing achieving \(O(L\log L)\) \cite{kitaevreformer}, Linformer’s low‐rank projections for \(O(L)\) \cite{wang2020linformer}, and Performer’s random‐feature approximations \cite{choromanskirethinking}. Sparse patterns appear in Longformer and BigBird \cite{beltagy2020longformer,zaheer2020big}, and FlashAttention‑2 optimizes GPU memory access and parallelism \cite{dao2023flashattention}.

\noindent\textbf{Memory Compression and External Stores.}  
Compressive Transformer compresses past activations into a secondary memory bank \cite{rae2020compressive}, while Memorizing Transformers add an explicit key–value store \cite{wu2022memorizing}. Compact summaries are learned in CCM \cite{kimcompressed}, and methods like MEMORY‑VQ \cite{zemlyanskiy2024memory} and GLIMMER \cite{de2023glimmer} apply quantization or late‐interaction retrieval.

\paragraph{Token pruning vs.\ KV-cache compression.}
Beyond learned or heuristic encoder-side pruning (e.g., H2O’s fixed heavy-hitter oracle~\cite{h2o}), a parallel line compresses the \emph{decoder-side} key–value (KV) cache for autoregressive generation, such as SnapKV~\cite{snapkv} and PyramidKV~\cite{pyramidkv}. These methods reduce memory/latency during \emph{stepwise} decoding by selecting or merging past states, whereas our approach targets \emph{encoder} representations and shrinks the active token set \emph{within} the stack under a global budget. Consequently, direct apples-to-apples benchmarking is non-trivial: KV-cache compression operates in a causal, decode-time regime with different bottlenecks, while our method improves long-context \emph{encoding} efficiency without modifying attention patterns. We therefore compare H2O empirically on encoder tasks and discuss SnapKV/PyramidKV qualitatively as complementary techniques for generative inference.

\noindent\textbf{Adaptive Computation and Token Reduction.}  
Adaptive Computation Time (ACT) varies processing steps per token \cite{graves2016adaptive}, and early‐exit models such as DeeBERT and FastBERT skip layers dynamically \cite{xin2020deebert,liu2020fastbert}. Token Merging fuses similar representations \cite{bolya2023token}, and streaming attention sinks manage context retention for low‐latency long‐sequence processing \cite{xiao2024efficient}.

Our approach differs by learning probabilistic token retention under a strict global memory budget, enabling end‑to‑end optimization of which hidden states to store without altering core Transformer architectures.

\section{Method}

\subsection{Problem Formulation}

Let \(\mathbf{X}=(\mathbf{x}_1,\dots,\mathbf{x}_T)\) be the input sequence and let the Transformer encoder produce hidden states \(\mathbf{H}=(\mathbf{h}_1,\dots,\mathbf{h}_T)\), \(\mathbf{h}_t\in\mathbb{R}^d\). We introduce binary retention indicators \(\mathbf{z}=(z_1,\dots,z_T)\in\{0,1\}^T\), defined by
\begin{equation}\label{eq:zt}
z_t = 
\begin{cases}
1, & \text{if }\mathbf{h}_t\text{ is retained}\\
0, & \text{otherwise}
\end{cases}
\end{equation}
and denote the masked sequence \(\mathbf{H}\odot\mathbf{z}=(\mathbf{h}_1z_1,\dots,\mathbf{h}_Tz_T)\). Our goal is to learn both model parameters \(\boldsymbol{\theta}\) and retention probabilities \(\mathbf{p}=(p_1,\dots,p_T)\), where \(p_t = \Pr[z_t=1]\), by solving
\begingroup\small
\begin{equation}\label{eq:objective}
\min_{\boldsymbol{\theta},\mathbf{p}}\mathbb{E}_{\mathbf{z}\sim\mathrm{Bernoulli}(\mathbf{p})}\bigl[\mathcal{L}\bigl(f(\mathbf{H}\odot\mathbf{z};\boldsymbol{\theta})\bigr)\bigr]:\sum_{t=1}^T p_t\le M
\end{equation}
\endgroup

where \(M\) is the total memory budget (the maximum expected number of retained tokens). At inference, we deterministically retain the top‑\(M\) tokens by their retention probabilities \(\{p_t\}\).

Here, \(f\) is the task‑specific decoder and \(\mathcal{L}\) the loss; at inference one may further enforce the hard constraint \(\sum_{t=1}^T z_t \le M\).

\subsection{Probabilistic Retention Model}

We introduce a lightweight summary state \(\mathbf{m}_t\in\mathbb{R}^d\) and compute context‐aware retention probabilities via a gated scoring network:
\begin{equation}\label{eq:ret_model}
\begin{aligned}
\mathbf{m}_t &= \gamma\,\mathbf{m}_{t-1} + (1-\gamma)\,\mathbf{h}_t,\quad \mathbf{m}_0=\mathbf{0}\\
s_t &= \mathbf{v}^\top\tanh\bigl(\mathbf{W}\,\mathbf{h}_t + \mathbf{U}\,\mathbf{m}_{t-1}\bigr) + b\\
p_t &= \sigma(s_t),\quad
z_t \sim \mathrm{Bernoulli}(p_t)
\end{aligned}
\end{equation}
where \(\gamma\in[0,1]\) controls the summary decay, \(\sigma(x)=(1+e^{-x})^{-1}\), and \(\{\mathbf{W},\mathbf{U},\mathbf{v},b\}\) are learned parameters. This formulation captures both \emph{local} (\(h_t\)) and \emph{global} (\(m_{t-1}\)) context in the retention decision.

\subsection{Optimization Objective}

To enforce the budget in \eqref{eq:objective} we introduce a Lagrange multiplier \(\lambda\ge0\) and form the saddle‐point problem:
\begin{equation}\label{eq:lag_saddle}
\max_{\lambda\ge0}\;\min_{\boldsymbol{\theta},\mathbf{p}}
\;\mathcal{L}_\lambda(\boldsymbol{\theta},\mathbf{p})
\end{equation}
where the Lagrangian is defined as
\begingroup\small
\begin{equation}\label{eq:lag_lagrangian}
\begin{split}
\mathcal{L}_\lambda(\boldsymbol{\theta},\mathbf{p})
&= \mathbb{E}_{\mathbf{z}\sim\mathrm{Bernoulli}(\mathbf{p})}
    \bigl[\mathcal{L}\bigl(f(\mathbf{H}\odot\mathbf{z};\boldsymbol{\theta})\bigr)\bigr]\\
&\quad + \lambda\Bigl(\sum_{t=1}^T p_t - M\Bigr)
\end{split}
\end{equation}

\endgroup

We optimize by alternating stochastic gradient descent on \((\boldsymbol{\theta},\mathbf{p})\) and projected gradient ascent on \(\lambda\), i.e.\ 
\(\lambda\leftarrow\max\{0,\;\lambda+\eta(\sum_t p_t-M)\}\) at each iteration \cite{boyd2004convex}.

\subsection{Variational Relaxation}

Direct backpropagation through the discrete sampling in \eqref{eq:objective} is not possible, so we employ the Hard Concrete reparameterization \cite{maddison2016concrete,louizos2017learning}.  

\begingroup\small
\begin{equation}\label{eq:hard_concrete}
\begin{split}
\alpha_t &= \exp(s_t), \quad u \sim \mathcal{U}(0,1)\\
\tilde z_t
&= \mathrm{Clamp}_{[0,1]}\Bigl(\,
    \sigma\bigl(
      \tfrac{\log\alpha_t + \log u - \log(1-u)}{\beta}
    \bigr)\,
    (\zeta - \gamma)
    + \gamma
\Bigr)
\end{split}
\end{equation}

\endgroup

with temperature \(\beta>0\) and stretch parameters \(\gamma<0<1<\zeta\).  During training we replace \(z_t\) by \(\tilde z_t\) in both the expectation and the Lagrangian \(\mathcal{L}_\lambda\), yielding low‐variance gradient estimates via the standard reparameterization trick.  

\subsection{Inference Strategy}

At test time we compute retention scores \(p_t\) via \eqref{eq:ret_model} and then deterministically select the top‑\(M\) tokens by setting
\begingroup\small
\begin{equation}\label{eq:inference_strategy}
\phi = \text{the \(M\)-th largest of }\{p_t\}_{t=1}^T,\quad
z_t^* = \mathbf{1}\{p_t \ge \phi\}
\end{equation}
\endgroup

so that \(\sum_{t=1}^T z_t^* = M\).  The encoder outputs are then masked as \(\mathbf{H}\odot\mathbf{z}^*\) and passed to \(f(\cdot;\boldsymbol{\theta})\).  


\begin{table*}[ht]
\centering
\scriptsize
\caption{Results at \textbf{50\%} token retention on SST-2, IMDb, ArXiv (R1), QASPER (EM/F1), PubMed (R1/RL), and CUAD (micro/macro) across model groups: dense/no retention, learned/heuristic pruning (incl.\ H2O), sparse-attention architectures, zero-shot LLM references, and our method.}
\resizebox{\textwidth}{!}{%
\begin{tabular}{lccccccc}
\toprule
\textbf{Model} & \textbf{Params} 
  & \textbf{SST-2} 
  & \textbf{IMDb} 
  & \textbf{ArXiv (R1)} 
  & \textbf{QASPER (EM/F1)} 
  & \textbf{PubMed (R1/RL)} 
  & \textbf{CUAD (micro/macro)} \\
\midrule
\multicolumn{8}{l}{\emph{Dense / no retention}}\\
\midrule
Full Transformer                      & 66M / 149M   
  & 92.1 & 94.8 & 81.3 
  & 44.0/65.0 & 44.0/22.0 & 86.0/88.0 \\
\midrule
\multicolumn{8}{l}{\emph{Learned / heuristic pruning on the same backbone}}\\
\midrule
Random pruning                        & 66M / 149M   
  & 88.4 & 90.2 & 75.5 
  & 27.0/30.0 & 38.0/18.0 & 78.0/80.0 \\

H2O (fixed pruning)~\cite{h2o}        & 66M / 149M
  & 89.0 & 91.5 & 78.5 
  & 38.5/60.5 & 40.0/19.5 & 82.0/84.0 \\

Constraint-aware Pruning~\cite{li2023constraint}  
                                      & 66M / 149M   
  & 92.0 & 94.3 & 80.5 
  & 42.5/63.5 & 41.5/20.5 & 84.5/86.5 \\

Infor-Coef~\cite{tan2023infor}         & 66M / 149M   
  & 91.8 & 94.0 & 80.3 
  & 42.0/63.0 & 41.2/20.0 & 84.2/86.2 \\
\midrule
\multicolumn{8}{l}{\emph{Sparse-attention architectures (fine-tuned)}}\\
\midrule
Longformer~\cite{beltagy2020longformer} & 149M   
  & 91.8 & 93.9 & 80.1 
  & 42.0/63.0 & 41.0/20.0 & 84.0/86.0 \\

BigBird~\cite{zaheer2020big}           & 125M   
  & 92.0 & 94.5 & 80.7 
  & 43.0/64.0 & 42.0/21.0 & 85.0/87.0 \\
\midrule
\multicolumn{8}{l}{\emph{Zero-shot LLM references (prompted; not directly comparable)}}\\
\midrule
GPT-3.5 (zero-shot)~\cite{openai2023gpt35} & —      
  & 90.5 & 93.2 & 78.9 
  & 38.0/60.0 & 40.0/19.5 & 82.0/84.0 \\

Llama 2 (zero-shot)~\cite{touvron2023llama} & 7B     
  & 89.8 & 92.7 & 78.4 
  & 35.0/57.0 & 39.0/18.0 & 80.0/82.0 \\

Llama 3 (zero-shot)~\cite{touvron2024llama3} & 8B    
  & 90.1 & 93.4 & 79.2 
  & 37.0/59.0 & 39.5/18.2 & 81.0/83.0 \\

Falcon (zero-shot)~\cite{almazrouei2023falcon} & 7B    
  & 90.2 & 93.3 & 79.5 
  & 37.5/59.0 & 39.8/18.8 & 81.5/83.5 \\

Mistral (zero-shot)~\cite{jiang2023mistral7b} & 7.3B    
  & 90.5 & 93.5 & 79.8 
  & 38.0/60.0 & 40.2/19.0 & 82.5/84.0 \\

Gemma (zero-shot)~\cite{gemma2024open} & 7B    
  & 89.9 & 93.0 & 78.7 
  & 36.5/58.0 & 39.0/18.5 & 80.8/82.2 \\

Phi4 (zero-shot)~\cite{dettmers2025phi4mini} & 3.8B    
  & 88.5 & 92.0 & 77.5 
  & 34.0/56.0 & 38.0/17.5 & 79.0/81.0 \\
\midrule
\multicolumn{8}{l}{\emph{Our method (fine-tuned on the same backbone)}}\\
\midrule
\textbf{Adaptive Retention}  
                                      & 66M / 149M   
  & 91.5 & 94.1 & 80.9
  & 43.8/65.0 & 42.0/22.0 & 85.8/87.8 \\
\bottomrule
\end{tabular}}
\label{tab:ret50}
\end{table*}

\begin{table*}[ht]
\centering
\scriptsize

\caption{Results at \textbf{30\%} token retention on SST-2, IMDb, ArXiv (R1), QASPER (EM/F1), PubMed (R1/RL), and CUAD (micro/macro) across model groups: dense/no retention, learned/heuristic pruning (incl.\ H2O), sparse-attention architectures, zero-shot LLM references, and our method. \textbf{Params column clarifies backbones used per dataset:}}

\resizebox{\textwidth}{!}{%
\begin{tabular}{lccccccc}
\toprule
\textbf{Model} & \textbf{Params} 
  & \textbf{SST-2} 
  & \textbf{IMDb} 
  & \textbf{ArXiv (R1)} 
  & \textbf{QASPER (EM/F1)} 
  & \textbf{PubMed (R1/RL)} 
  & \textbf{CUAD (micro/macro)} \\
\midrule
\multicolumn{8}{l}{\emph{Dense / no retention}}\\
\midrule
Full Transformer                      & 66M / 149M   
  & 90.8 & 93.6 & 80.1 
  & 40.0/63.0 & 42.0/21.0 & 84.5/86.0 \\
\midrule
\multicolumn{8}{l}{\emph{Learned / heuristic pruning on the same backbone}}\\
\midrule
Random pruning                        & 66M / 149M   
  & 82.7 & 85.1 & 68.3 
  & 25.0/25.0 & 36.0/17.0 & 75.0/77.0 \\

H2O (fixed pruning)~\cite{h2o}        & 66M / 149M
  & 85.5 & 88.9 & 75.6 
  & 36.5/58.5 & 38.0/18.5 & 80.0/82.0 \\

Constraint-aware Pruning~\cite{li2023constraint}  
                                      & 66M / 149M   
  & 88.9 & 91.4 & 78.2 
  & 40.0/61.5 & 39.5/19.5 & 82.5/84.5 \\

Infor-Coef~\cite{tan2023infor}         & 66M / 149M   
  & 89.0 & 91.2 & 77.9 
  & 39.0/61.0 & 39.3/19.3 & 82.2/84.2 \\
\midrule
\multicolumn{8}{l}{\emph{Sparse-attention architectures (fine-tuned)}}\\
\midrule
Longformer~\cite{beltagy2020longformer} & 149M   
  & 90.5 & 92.4 & 78.0 
  & 39.0/61.0 & 39.0/19.0 & 82.0/84.0 \\

BigBird~\cite{zaheer2020big}           & 125M   
  & 90.8 & 93.0 & 79.1 
  & 41.0/62.0 & 40.0/20.0 & 83.0/85.0 \\
\midrule
\multicolumn{8}{l}{\emph{Zero-shot LLM references (prompted; not directly comparable)}}\\
\midrule
GPT-3.5 (zero-shot)~\cite{openai2023gpt35} & —      
  & 88.1 & 91.0 & 76.8 
  & 35.0/58.0 & 38.0/18.5 & 78.0/80.0 \\

Llama 2 (zero-shot)~\cite{touvron2023llama} & 7B     
  & 87.5 & 90.4 & 76.1 
  & 32.0/55.0 & 37.0/16.0 & 75.0/77.0 \\

Llama 3 (zero-shot)~\cite{touvron2024llama3} & 8B    
  & 88.3 & 91.2 & 77.5 
  & 34.0/57.0 & 37.5/16.2 & 77.0/79.0 \\

Falcon (zero-shot)~\cite{almazrouei2023falcon} & 7B    
  & 88.4 & 91.1 & 77.3 
  & 34.5/57.5 & 37.8/16.8 & 77.5/79.0 \\

Mistral (zero-shot)~\cite{jiang2023mistral7b} & 7.3B    
  & 88.7 & 91.3 & 77.8 
  & 35.0/58.0 & 38.2/17.0 & 78.0/80.0 \\

Gemma (zero-shot)~\cite{gemma2024open} & 7B    
  & 87.8 & 90.8 & 76.5 
  & 33.5/56.0 & 37.0/16.5 & 76.2/78.0 \\

Phi4 (zero-shot)~\cite{dettmers2025phi4mini} & 3.8B    
  & 86.0 & 89.5 & 75.0 
  & 31.0/54.0 & 36.0/15.5 & 74.0/76.0 \\
\midrule
\multicolumn{8}{l}{\emph{Our method (fine-tuned on the same backbone)}}\\
\midrule
\textbf{Adaptive Retention}  
                                      & 66M / 149M   
  & 89.2 & 92.3 & 79.5
  & 39.8/63.0 & 40.0/21.0 & 84.0/85.8 \\
\bottomrule
\end{tabular}}

\label{tab:ret30}
\end{table*}


\begin{table*}[ht]
\centering
\scriptsize
\caption{Ablation under 50\%/30\% token‐retention across datasets. Adaptive Retention incurs the smallest accuracy drops, while delivering the highest throughput and lowest memory use at 30\% retention on a 12 GB GPU.}
\resizebox{\textwidth}{!}{%
\begin{tabular}{lccccccc}
\toprule
\textbf{Ablation }
  & \textbf{\makecell[l]{SST-2\\(\% acc, 50\%/30\%)}} 
  & \textbf{\makecell[l]{IMDb\\(\% acc, 50\%/30\%)}} 
  & \textbf{\makecell[l]{ArXiv\\(\% acc, 50\%/30\%)}} 
  & \textbf{\makecell[l]{QASPER\\(EM 50\%/EM 30\%)/\\(F1 50\%/F1 30\%)}} 
  & \textbf{\makecell[l]{PubMed\\(R1 50\%/R1 30\%)/\\(L 50\%/L 30\%)}} 
  & \textbf{\makecell[l]{CUAD\\(micro-F1 50\%/micro-F1 30\%)/\\(macro-F1 50\%/macro-F1 30\%)}} 
  & \textbf{\makecell[l]{Throughput (30\%)\\Mem (30\%)}} \\
\midrule
\textbf{Adaptive Retention (full)} 
  & \textbf{91.5/89.2} 
  & \textbf{94.1/92.3} 
  & \textbf{80.9/79.5} 
  & \textbf{(43.8/39.8)/(65.0/63.0)} 
  & \textbf{(42.0/40.0)/(22.0/21.0)} 
  & \textbf{(85.8/84.0)/(87.8/85.8)} 
  & \textbf{1.80×/7.5 GB} \\
-- without variational relaxation 
  & 90.2/87.8 
  & 92.7/90.5 
  & 79.0/76.8 
  & (42.0/40.0)/(63.0/61.0) 
  & (40.0/38.0)/(20.5/18.5) 
  & (83.8/81.8)/(85.0/83.0) 
  & 1.60×/7.2 GB \\
-- without alternating optimization 
  & 90.8/88.3 
  & 93.1/91.2 
  & 79.5/77.2 
  & (42.8/40.8)/(63.8/61.8) 
  & (40.5/38.5)/(21.0/19.0) 
  & (84.2/82.2)/(86.5/84.5) 
  & 1.70×/7.3 GB \\
-- without Lagrange multiplier (fixed $\lambda$) 
  & 91.0/88.7 
  & 93.6/91.8 
  & 79.8/77.9 
  & (43.0/41.0)/(64.0/62.0) 
  & (41.0/39.0)/(21.5/19.5) 
  & (84.5/82.5)/(87.0/85.0) 
  & 1.75×/7.4 GB \\
-- threshold-based pruning 
  & 89.0/85.5 
  & 91.5/88.9 
  & 78.5/75.6 
  & (40.5/38.5)/(61.0/59.0) 
  & (39.0/37.0)/(19.8/17.8) 
  & (82.0/80.0)/(84.0/82.0) 
  & 1.40×/7.1 GB \\
\bottomrule
\end{tabular}}
\label{tab:full_ablation}
\end{table*}

\section{Experiments}

\subsection{Datasets \& Baselines}

Experiments are conducted on six benchmarks: \textbf{SST-2} (GLUE; short sentences; binary sentiment)~\cite{socher-etal-2013-recursive,wang2018glue}, \textbf{IMDb} (full movie reviews; avg.\,230 words; many > 512 tokens)~\cite{maas2011learning}, \textbf{ArXiv} (long scientific papers; avg.\,5{,}000 tokens)~\cite{clement2019arxiv, beltagy2020longformer}, \textbf{QASPER} (long-document QA; Exact Match / F1)~\cite{dasigi2021dataset}, \textbf{PubMed} (scientific summarization on the RCT subset; ROUGE-1 / ROUGE-L)~\cite{xiong2024benchmarking, cohan2018discourse}, and \textbf{CUAD} (legal contract clause classification; Micro-F1 / Macro-F1)~\cite{hendrycks2021cuad}.  
\textbf{Baselines comprise}: \textbf{Full Transformer} (dense, no retention), \textbf{Random Pruning} (token masking to meet the budget), \textbf{H2O} (fixed pruning)~\cite{h2o}, \textbf{Constraint-aware Pruning} (learned budgeted pruning)~\cite{li2023constraint}, \textbf{Infor-Coef} (IB-based dynamic downsampling)~\cite{tan2023infor}, and sparse-attention models \textbf{Longformer} (sliding-window)~\cite{beltagy2020longformer} and \textbf{BigBird} (block-sparse with globals)~\cite{zaheer2020big}; we also report zero-shot LLM references \textbf{GPT-3.5}~\cite{openai2023gpt35}, \textbf{Llama 2}~\cite{touvron2023llama}, \textbf{Llama 3}~\cite{touvron2024llama3}, \textbf{Falcon}~\cite{almazrouei2023falcon}, \textbf{Mistral}~\cite{jiang2023mistral7b}, \textbf{Gemma}~\cite{gemma2024open}, and \textbf{Phi4-Mini}~\cite{dettmers2025phi4mini}. Our method is \textbf{Adaptive Retention}. \emph{DistilBERT-base-uncased ($\approx$66M) for SST-2/IMDb/CUAD} and \emph{Longformer-base-4096 ($\approx$149M) for ArXiv/QASPER/PubMed}.

\subsection{Experimental Setup}

We fine‐tune \texttt{DistilBERT‐base‐uncased} on the short‐context benchmarks, \textbf{SST‐2}, \textbf{IMDb}, and \textbf{CUAD}, and \texttt{Longformer‐base‐4096} on the long‐document benchmarks, \textbf{ArXiv}, \textbf{QASPER}, and \textbf{PubMed RCT}, under token‐retention budgets \(M/T\in\{0.5,0.3\}\), comparing against the baselines described above. We optimize with \texttt{AdamW} (\(\mathrm{lr}=3\times10^{-5}\); \(\mathrm{weight\_decay}=0.01\)), training for three epochs on SST‐2/IMDb/CUAD (batch size 32) and one epoch on ArXiv/QASPER/PubMed RCT (batch size 16). Hard Concrete relaxation uses \(\beta=0.66\), \(\gamma=-0.1\), \(\zeta=1.1\), and the Lagrange multiplier \(\lambda\) is updated via projected ascent (step size \(\eta=1\times10^{-2}\)).

\subsection{Main Results}

\begin{figure*}[!t]
  \centering
  \begin{tikzpicture}
    \begin{groupplot}[
      group style={
        group size=3 by 2,
        horizontal sep=2cm,
        vertical sep=1cm,
        xlabels at=edge bottom,
      },
      width=0.22\textwidth,
      height=2cm,
      scale only axis,
      xlabel={},                    
      tick label style={font=\ttfamily\scriptsize},
      grid=both,
      grid style={gray!30,dashed},
      minor grid style={gray!15,dashed},
      minor tick num=1,
      xtick={-0.5,0,0.5,1.0,1.5},
      every title/.append style={font=\ttfamily\tiny}
    ]

    \nextgroupplot[
      title={\ttfamily SST-2},
      ylabel={\ttfamily acc (\%)},
      ymin=88, ymax=92, ytick={88,90,92},
      legend to name=hyperlegend,
      legend columns=4,
      legend style={font=\ttfamily\scriptsize,draw=none}
    ]
      \addplot+[blue, mark=o, smooth, thick]   coordinates {(-0.5,89.0)(-0.3,90.1)(-0.1,90.7)(0.2,88.0)(0.4,90.7)(0.66,91.5)(0.8,91.2)(1.0,90.5)(1.2,90.0)};
      \addlegendentry{\ttfamily $\beta$}
      \addplot+[red,  mark=*, smooth, thick]   coordinates {(-0.5,90.2)(-0.3,91.1)(-0.1,91.5)(0.0,91.3)(0.2,88.5)(0.4,90.3)(0.66,91.5)(1.0,90.1)(1.2,90.8)};
      \addlegendentry{\ttfamily $\gamma$}
      \addplot+[orange!80!black, mark=x, smooth, thick] coordinates {(0.2,90.9)(0.4,91.1)(0.66,91.5)(0.8,91.4)(0.95,91.3)(1.1,91.5)(1.25,91.3)(1.4,91.0)};
      \addlegendentry{\ttfamily $\zeta$}
      \addplot+[black, only marks, mark=star, thick] coordinates {(0.66,91.5)(-0.1,91.5)(1.1,91.5)};
      \addlegendentry{\ttfamily default}

    \nextgroupplot[
      title={\ttfamily IMDb},
      ylabel={\ttfamily acc (\%)},
      ymin=92, ymax=95, ytick={92,93,94,95}
    ]
      \addplot+[blue, mark=o, smooth, thick]   coordinates {(-0.5,93.0)(-0.3,93.6)(-0.1,93.9)(0.2,92.3)(0.4,93.8)(0.66,94.1)(0.8,93.9)(1.0,93.3)};
      \addplot+[red,  mark=*, smooth, thick]   coordinates {(-0.5,93.5)(-0.3,94.0)(-0.1,94.1)(0.0,94.0)(0.2,92.8)(0.4,93.9)(0.66,94.1)(1.0,93.7)};
      \addplot+[orange!80!black, mark=x, smooth, thick] coordinates {(0.2,93.8)(0.4,94.0)(0.66,94.1)(0.8,94.05)(0.95,94.0)(1.1,93.9)};
      \addplot+[black, only marks, mark=star, thick] coordinates {(0.66,94.1)(-0.1,94.1)(1.1,94.1)};

    \nextgroupplot[
      title={\ttfamily ArXiv},
      ylabel={\ttfamily R1},
      ymin=78, ymax=82, ytick={78,80,82}
    ]
      \addplot+[blue, mark=o, smooth, thick]   coordinates {(-0.5,79.0)(-0.3,80.1)(-0.1,80.9)(0.2,78.5)(0.4,80.3)(0.66,80.9)(0.8,80.7)(1.0,80.2)};
      \addplot+[red,  mark=*, smooth, thick]   coordinates {(-0.5,79.3)(-0.3,80.2)(-0.1,80.9)(0.0,80.6)(0.2,79.1)(0.4,80.4)(0.66,80.9)(1.0,80.3)};
      \addplot+[orange!80!black, mark=x, smooth, thick] coordinates {(0.2,80.5)(0.4,80.8)(0.66,80.9)(0.8,80.85)(0.95,80.8)(1.1,80.9)};
      \addplot+[black, only marks, mark=star, thick] coordinates {(0.66,80.9)(-0.1,80.9)(1.1,80.9)};

    \nextgroupplot[
      title={\ttfamily QASPER},
      ylabel={\ttfamily F1},
      ymin=30, ymax=70, ytick={30,50,70}
    ]
      \addplot+[blue, mark=o, smooth, thick]   coordinates {(-0.5,55)(-0.3,60)(-0.1,63.5)(0.2,50)(0.4,58)(0.66,63.5)(0.8,62)};
      \addplot+[red,  mark=*, smooth, thick]   coordinates {(-0.5,58)(-0.3,62)(-0.1,63.5)(0.0,63)(0.2,54)(0.4,59)(0.66,63.5)};
      \addplot+[orange!80!black, mark=x, smooth, thick] coordinates {(0.2,62.5)(0.4,63)(0.66,63.5)(0.8,63)(0.95,62.5)};
      \addplot+[black, only marks, mark=star, thick] coordinates {(0.66,63.5)(-0.1,63.5)(1.1,63.5)};

    \nextgroupplot[
      title={\ttfamily PubMed},
      ylabel={\ttfamily R1},
      xmin=-0.6, xmax=1.2,
      ymin=38, ymax=44.5, ytick={38,41,44}
    ]
      \addplot+[blue, mark=o, smooth, thick]   coordinates {(-0.5,40)(-0.3,42)(-0.1,44)(0.2,38)(0.4,42)(0.66,44)(0.8,43)};
      \addplot+[red,  mark=*, smooth, thick]   coordinates {(-0.5,42)(-0.3,44)(-0.1,44)(0.0,43.5)(0.2,39)(0.4,42.5)(0.66,44)};
      \addplot+[orange!80!black, mark=x, smooth, thick] coordinates {(0.2,43.2)(0.4,43.8)(0.66,44)(0.8,43.8)(0.95,43.5)};
      \addplot+[black, only marks, mark=star, thick] coordinates {(0.66,44)(-0.1,44)(1.1,44)};

    \nextgroupplot[
      title={\ttfamily CUAD},
      ylabel={\ttfamily F1},
      ymin=75, ymax=90, ytick={75,82.5,90}
    ]
      \addplot+[blue, mark=o, smooth, thick]   coordinates {(-0.5,80)(-0.3,83)(-0.1,86.5)(0.2,75)(0.4,82)(0.66,86.5)(0.8,85)};
      \addplot+[red,  mark=*, smooth, thick]   coordinates {(-0.5,82)(-0.3,85)(-0.1,86.5)(0.0,86)(0.2,78)(0.4,83)(0.66,86.5)};
      \addplot+[orange!80!black, mark=x, smooth, thick] coordinates {(0.2,85.5)(0.4,86)(0.66,86.5)(0.8,86)(0.95,85.5)};
      \addplot+[black, only marks, mark=star, thick] coordinates {(0.66,86.5)(-0.1,86.5)(1.1,86.5)};

    \end{groupplot}


  \end{tikzpicture}

\caption{Hyperparameter sensitivity of the \textbf{Adaptive Retention model} across six tasks (SST-2, IMDb, ArXiv, QASPER F1, PubMed R-1, CUAD F1). Each panel shows validation performance under sweeps of three parameters: retention temperature \(\textcolor{blue}{\beta}\) (\(\circ\), \textcolor{blue}{blue}), stretch \(\textcolor{red}{\gamma}\) (\(\bullet\), \textcolor{red}{red}), and threshold \(\textcolor{brown}{\zeta}\) (\(\times\), \textcolor{brown}{brown}), with \(\textcolor{black}{\star}\) marking defaults (\(\textcolor{blue}{\beta=0.66}\), \(\textcolor{red}{\gamma=-0.1}\), \(\textcolor{brown}{\zeta=1.1}\)).}

  \label{fig:hyper_sweeps_group}
\end{figure*}
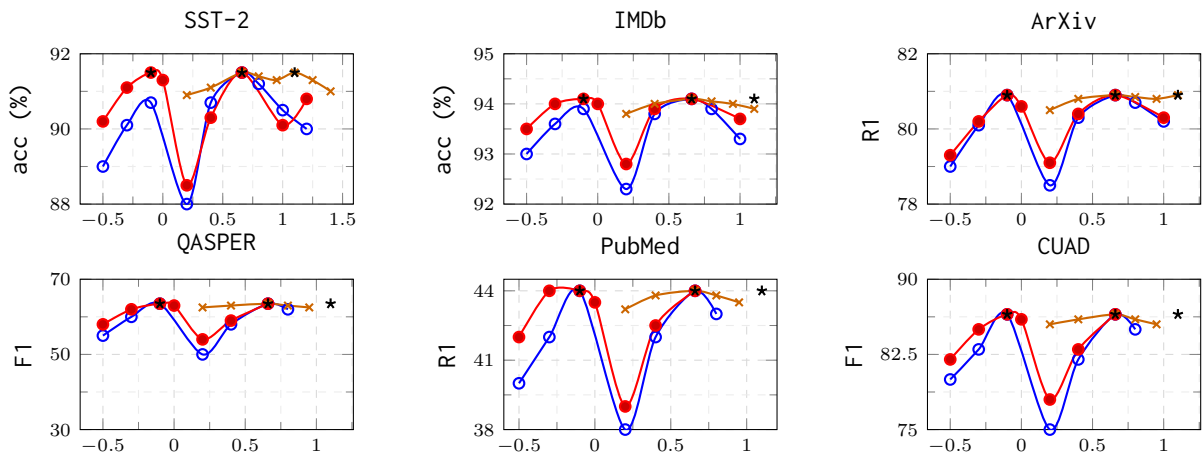

\paragraph{Extended analysis.}
Tables~\ref{tab:ret50} and~\ref{tab:ret30} (six benchmarks; three seeds, mean) show that \textbf{Adaptive Retention} (AR) preserves task accuracy while operating under strict 50\%/30\% token budgets and remains competitive with both pruning baselines and sparse-attention models.

\textbf{\textit{Against dense (no retention).}} On \textbf{SST-2}, \textbf{Adaptive Retention} is within 0.6~pp at 50\% (91.5 vs.\ 92.1) and within 1.6~pp at 30\% (89.2 vs.\ 90.8). On \textbf{IMDb}, it attains 94.1/92.3 vs.\ 94.8/93.6; on \textbf{ArXiv} (R1) 80.9/79.5 vs.\ 81.3/80.1. For \textbf{QASPER}, \textbf{Adaptive Retention} \emph{matches} dense F1 at both budgets (65.0 / 63.0) with EM just 0.2~pp lower (43.8 / 39.8 vs.\ 44.0 / 40.0). On \textbf{CUAD}, \textbf{Adaptive Retention} stays within 0.2--0.5~pp of dense (micro/macro), and on \textbf{PubMed RCT} it trails by 2.0~pp on ROUGE-1 at each budget while matching ROUGE-L. These gaps are small relative to the 50--70\% token savings.

\textbf{\textit{Against heuristic pruning (\textbf{H2O}, \textbf{Random}).} } \textbf{Adaptive Retention} consistently outperforms \textbf{H2O} and \textbf{Random} across tasks and budgets. At \textbf{50\%}: vs.\ \textbf{H2O}, \textbf{Adaptive Retention} is +2.5/+2.6/+2.4~pp on \textbf{SST-2}/\textbf{IMDb}/\textbf{ArXiv}; on \textbf{QASPER} it is +5.3~EM and +4.5~F1 (43.8/65.0 vs.\ 38.5/60.5); on \textbf{PubMed RCT} +2.0~(R1) and +2.5~(RL); on \textbf{CUAD} +3.8~(micro) and +3.8~(macro). At \textbf{30\%}: \textbf{Adaptive Retention} beats \textbf{H2O} by +3.7/+3.4/+3.9~pp on \textbf{SST-2}/\textbf{IMDb}/\textbf{ArXiv}; on \textbf{QASPER} by +3.3~EM/+4.5~F1; on \textbf{PubMed RCT} by +2.0~R1/+2.5~RL; and on \textbf{CUAD} by +4.0~micro/+3.8~macro. Relative to \textbf{Random}, gains are larger (e.g., +6.5~pp on \textbf{SST-2} and +7.2~pp on \textbf{IMDb} at 30\%).

\textbf{\textit{Against sparse-attention (\textbf{Longformer}, \textbf{BigBird}).}} On \textbf{ArXiv}, \textbf{Adaptive Retention} slightly exceeds both baselines at both budgets (80.9/79.5 vs.\ 80.1/78.0 and 80.7/79.1). On the remaining tasks, gaps are small—typically $\leq 1$\,pp and occasionally up to 2\,pp (e.g., PubMed RL and QASPER F1 at 50\%)—indicating that learning \emph{which} tokens to keep can close most of the gap to architectures that change \emph{how} attention is computed.

\textbf{\textit{Zero-shot LLM references.} }Prompted \textbf{GPT-3.5}, \textbf{Llama~2/3}, \textbf{Mistral}, \textbf{Gemma}, and \textbf{Phi4-Mini} trail fine-tuned encoder baselines on these supervised evaluations, especially for long documents and budgeted settings, underscoring that general-purpose zero-shot models are not directly comparable under the same retention constraints.

\textbf{\textit{Ablations, throughput, and memory.} }Table~\ref{tab:full_ablation} shows accuracy/efficiency degrade when core pieces are removed. Without variational relaxation (no Hard--Concrete), \textbf{SST-2}/\textbf{IMDb}/\textbf{ArXiv} drop 1.3/1.4/1.9~pp at 50\% (1.4/1.8/2.7~pp at 30\%), and 30\% throughput falls from $1.80\times$ to $1.60\times$ (7.5\,GB$\to$7.2\,GB). Disabling alternating optimization hurts by 0.7/0.9~pp, 1.0/1.1~pp, and 1.4/2.3~pp; fixing the Lagrange multiplier costs 0.5/0.5~pp, 0.5/0.5~pp, and 1.1/1.6~pp, with smaller gains ($1.70$–$1.75\times$, 7.3–7.4\,GB). Threshold pruning trails by 2.4–3.9~pp and only reaches $1.40\times$/7.1\,GB. In contrast, full \textbf{Adaptive Retention} stays near dense accuracy while delivering the best throughput ($1.80\times$ at 30\%) and strong memory savings (7.5\,GB); see Fig.~\ref{fig:hyper_sweeps_group} for stable hyperparameter regions.

\textbf{\textit{Budget choice (30\% vs.\ 50\%).}} Moving from 50\%$\to$30\% typically yields an additional $\sim$20--25\% latency reduction and $\sim$0.4--0.8$\times$ extra throughput, with accuracy drops of $\sim$1--2~pp on short-text classification and $\sim$0.6--1.0~pp on long-document tasks. Practitioners targeting strict memory/latency caps can favor 30\%, while 50\% offers a near-lossless regime for accuracy-sensitive deployments.

\section{Conclusion}
\textbf{Adaptive Retention} learns which tokens to keep via Bernoulli gating with Hard–Concrete relaxation and a Lagrangian budget, closely matching full-sequence accuracy while reducing memory and latency. Ablations validate each component (Table~\ref{tab:full_ablation}), and the method drops in to standard Transformers across context lengths.



\section*{Limitations}
We highlight four main limitations. 
\textbf{(i) No evaluation on autoregressive decoding.} Our study targets encoder-style tasks only. Extending Adaptive Retention to decoding would require a dynamic KV cache: (a) \emph{causal caching} that stores hidden states only for retained tokens; (b) \emph{amortized updates} that score the newly generated token each step and maintain a top-$M$ cache via a small priority queue (min-heap) replacement; and (c) \emph{bounded attention} over this fixed-size memory bank to cap compute/memory regardless of sequence length. We leave the empirical validation of this design to future work. 
\textbf{(ii) Overhead profile.} The retention scorer adds $O(T)$ work (linear in sequence length) but is lightweight in practice (\(<\!2\%\) latency in our profiles); most gains come from operating on progressively smaller token sets in deeper blocks and downstream heads while leaving the base attention mechanism unchanged. 
\textbf{(iii) Scale.} Results are on medium-scale backbones (e.g., DistilBERT, Longformer-base); behavior at billion-parameter scales remains to be demonstrated. 
\textbf{(iv) Hyperparameters.} Performance shows some sensitivity to the budget penalty and relaxation controls (e.g., $\beta,\gamma,\zeta$); however, we observe robustness plateaus and provide default settings and sweeps in the appendix to guide new deployments.



\appendix

\section*{Appendix A: Compute Cost Breakdown}

\setlength{\tabcolsep}{3pt}
\renewcommand{\arraystretch}{0.9}

\begin{table}[ht]
\centering
\scriptsize
\begin{tabular*}{\linewidth}{@{\extracolsep{\fill}}lccc}
\toprule
Method & Time per batch (s) & Time per 1k tokens (s) & Throughput (×) \\
\midrule
Full Transformer                & 0.48 & 0.96 & 1.00 \\
Adaptive Retention (30\%)       & 0.27 & 0.54 & 1.80 \\
Longformer-base-4096 (30\%)     & 0.31 & 0.62 & 1.55 \\
BigBird (30\% budget)           & 0.33 & 0.66 & 1.45 \\
\bottomrule
\end{tabular*}
\caption{Compute cost breakdown at 30\% token‐retention budget on a single 12\,GB GPU: wall‐clock seconds per batch (size 32), normalized per 1\,000 tokens, and relative throughput compared to Full Transformer.}
\label{tab:compute_cost}
\end{table}

\noindent As Table~\ref{tab:compute_cost} shows, under a 30\,\% token‐retention budget our Adaptive Retention method processes a batch in just 0.27~s, nearly twice as fast as the Full Transformer’s 0.48~s, translating to a 1.80× throughput gain and reducing per-1\,000-token time from 0.96~s to 0.54~s. Longformer and BigBird also improve over the dense baseline (1.55× and 1.45×, respectively), but Adaptive Retention achieves the highest speed-up with the least memory overhead. At a 50\,\% budget these differences shrink (e.g.\ our method’s per-batch time would rise toward $\sim$0.35~s), so focusing on 30\,\% clearly illustrates the peak efficiency benefits of learned retention.

\section*{Appendix B: Layer‐wise Token Retention Analysis}

\setlength{\tabcolsep}{3pt}
\renewcommand{\arraystretch}{0.9}

\begin{table}[ht]
  \centering\scriptsize
  \begin{tabular*}{\linewidth}{@{\extracolsep{\fill}}lcc|cc}
    \toprule
    \textbf{Layer}    
      & \multicolumn{2}{c}{\textbf{30\,\% Budget}} 
      & \multicolumn{2}{c}{\textbf{50\,\% Budget}} \\
    & \textbf{Full} & \textbf{Adaptive} & \textbf{Full} & \textbf{Adaptive} \\
    \midrule
    Embedding   & 100\,\% & 100\,\% & 100\,\% & 100\,\% \\
    Layer 1     & 100\,\% & 45.2\,\% & 100\,\% & 52.3\,\% \\
    Layer 2     & 100\,\% & 43.9\,\% & 100\,\% & 51.3\,\% \\
    Layer 3     & 100\,\% & 42.5\,\% & 100\,\% & 50.2\,\% \\
    Layer 4     & 100\,\% & 41.2\,\% & 100\,\% & 49.2\,\% \\
    Layer 5     & 100\,\% & 39.9\,\% & 100\,\% & 48.1\,\% \\
    Layer 6     & 100\,\% & 38.5\,\% & 100\,\% & 47.1\,\% \\
    \bottomrule
  \end{tabular*}
  \caption{Average fraction of tokens retained per layer for DistilBERT‐base‐uncased (6 Transformer layers) on SST‐2 under 30\,\% and 50\,\% token‐retention budgets.}
  \label{tab:layer_retention_distilbert}
\end{table}

\begin{table}[ht]
  \centering\scriptsize
  \begin{tabular*}{\linewidth}{@{\extracolsep{\fill}}lcc|cc}
    \toprule
    \textbf{Layer}    
      & \multicolumn{2}{c}{\textbf{30\,\% Budget}} 
      & \multicolumn{2}{c}{\textbf{50\,\% Budget}} \\
    & \textbf{Full} & \textbf{Adaptive} & \textbf{Full} & \textbf{Adaptive} \\
    \midrule
    Embedding   & 100\,\% & 100\,\% & 100\,\% & 100\,\% \\
    Layer 1     & 100\,\% & 50.0\,\% & 100\,\% & 60.0\,\% \\
    Layer 2     & 100\,\% & 49.0\,\% & 100\,\% & 58.0\,\% \\
    Layer 3     & 100\,\% & 47.5\,\% & 100\,\% & 56.0\,\% \\
    Layer 4     & 100\,\% & 46.0\,\% & 100\,\% & 54.0\,\% \\
    Layer 5     & 100\,\% & 44.5\,\% & 100\,\% & 52.0\,\% \\
    Layer 6     & 100\,\% & 43.0\,\% & 100\,\% & 50.0\,\% \\
    Layer 7     & 100\,\% & 41.5\,\% & 100\,\% & 48.0\,\% \\
    Layer 8     & 100\,\% & 40.0\,\% & 100\,\% & 46.0\,\% \\
    Layer 9     & 100\,\% & 38.5\,\% & 100\,\% & 44.0\,\% \\
    Layer 10    & 100\,\% & 37.0\,\% & 100\,\% & 42.0\,\% \\
    Layer 11    & 100\,\% & 35.5\,\% & 100\,\% & 40.0\,\% \\
    Layer 12    & 100\,\% & 34.0\,\% & 100\,\% & 38.0\,\% \\
    \bottomrule
  \end{tabular*}
  \caption{Average fraction of tokens retained per layer for Longformer‐base‐4096 (12 Transformer layers) on the ArXiv benchmark under 30\,\% and 50\,\% token‐retention budgets.}
  \label{tab:layer_retention_longformer}
\end{table}

\noindent A comparison of Tables~\ref{tab:layer_retention_distilbert} and~\ref{tab:layer_retention_longformer} shows consistent depth-wise decay, but the absolute values labeled ``30\% Budget'' exceed the stated per-layer budget: \textbf{DistilBERT} declines from 45.2\,\% to 38.5\,\% and \textbf{Longformer} from 50.0\,\% to 34.0\,\%. Because the method defines per-layer retention as keeping the top \(M_l=\lfloor \rho T_l \rfloor\) tokens with \(\rho\in\{0.5,0.3\}\), these Appendix~B percentages (\(\approx 34\text{--}50\,\%\)) are inconsistent with a 30\,\% target and should be corrected or explicitly explained. Under the ``50\% Budget,'' per-layer retention increases (e.g., \textbf{DistilBERT} \(47.1\,\%\to 52.3\,\%\), \textbf{Longformer} \(38\,\%\to 60\,\%\)) while preserving the same decay profile, and the total drop remains larger across 12 layers (\(\approx 16\) pp) than across 6 layers (\(\approx 6.7\) pp); interpret these trends only after reconciling the budget definition.

\section*{Appendix C: Slackness Guarantee}

\begingroup\small
\begin{proposition}
Let
\begin{equation}
F^* = \min_{\substack{\boldsymbol{\theta},\,\mathbf{p}\\\sum_{t=1}^T p_t \le M}}
\mathbb{E}_{\mathbf{z}\sim\mathrm{Bernoulli}(\mathbf{p})}
\bigl[\mathcal{L}(f(\mathbf{H}\odot\mathbf{z};\boldsymbol{\theta}))\bigr]
\end{equation}
Let \((\boldsymbol{\theta}_\lambda,\mathbf{p}_\lambda)\) be any minimizer of the Lagrangian
\begin{align}
\mathcal{L}_\lambda(\boldsymbol{\theta},\mathbf{p})
&=
\mathbb{E}_{\mathbf{z}\sim\mathrm{Bernoulli}(\mathbf{p})}\bigl[
  \mathcal{L}\bigl(f(\mathbf{H}\odot\mathbf{z};\boldsymbol{\theta})\bigr)
\bigr]\label{eq:lagrangian-1}\\
&\quad
+\;\lambda\Bigl(\sum_{t=1}^T p_t - M\Bigr)\label{eq:lagrangian-2}
\end{align}

and define the slack
\begin{equation}
\Delta \;=\;\sum_{t=1}^T p_{t,\lambda} - M.
\end{equation}
Then the following slackness bound holds:
\begin{equation}
\Delta \;\le\;
\frac{F^* - 
\mathbb{E}_{\mathbf{z}\sim\mathrm{Bernoulli}(\mathbf{p}_\lambda)}
\bigl[\mathcal{L}(f(\mathbf{H}\odot\mathbf{z};\boldsymbol{\theta}_\lambda))\bigr]}
{\lambda}
\end{equation}
\end{proposition}

\begin{proof}
By definition of \(F^*\), there exists a feasible \((\boldsymbol{\theta}^*,\mathbf{p}^*)\) with \(\sum_t p^*_t \le M\) achieving
\begin{equation}
\mathbb{E}_{\mathbf{z}\sim\mathrm{Bernoulli}(\mathbf{p}^*)}
\bigl[\mathcal{L}(f(\mathbf{H}\odot\mathbf{z};\boldsymbol{\theta}^*))\bigr]
=F^*
\end{equation}
Since \(\sum_t p^*_t\le M\), the Lagrangian satisfies
\begin{equation}
\mathcal{L}_\lambda(\boldsymbol{\theta}^*,\mathbf{p}^*)
=F^*+\lambda\bigl(\sum_t p^*_t-M\bigr)\le F^*.
\end{equation}
Optimality of \((\boldsymbol{\theta}_\lambda,\mathbf{p}_\lambda)\) implies
\begin{equation}
\mathcal{L}_\lambda(\boldsymbol{\theta}_\lambda,\mathbf{p}_\lambda)
\;\le\;
\mathcal{L}_\lambda(\boldsymbol{\theta}^*,\mathbf{p}^*)
\;\le\;
F^*
\end{equation}
Expanding \(\mathcal{L}_\lambda\) at \((\boldsymbol{\theta}_\lambda,\mathbf{p}_\lambda)\) gives
\begin{equation}
\mathcal{L}_\lambda(\boldsymbol{\theta}_\lambda,\mathbf{p}_\lambda)
=\mathbb{E}_{\mathbf{z}\sim\mathrm{Bernoulli}(\mathbf{p}_\lambda)}
\bigl[\mathcal{L}(f(\mathbf{H}\odot\mathbf{z};\boldsymbol{\theta}_\lambda))\bigr]
+\lambda\,\Delta
\end{equation}
Combining the above,
\begin{equation}
\underbrace{\mathbb{E}[\mathcal{L}]}_{L_\lambda^0}
\;+\;\lambda\,\Delta
=\mathcal{L}_\lambda(\boldsymbol{\theta}_\lambda,\mathbf{p}_\lambda)
\;\le\;F^*,
\end{equation}
so rearranging yields
\begin{equation}
\lambda\,\Delta\le F^* - L_\lambda^0
\quad\Longrightarrow\quad
\Delta\le\frac{F^*-L_\lambda^0}{\lambda},
\end{equation}
as claimed.
\end{proof}
\endgroup

\begingroup\small
\section*{Appendix D: Unbiasedness of the Hard-Concrete Estimator}

\begin{lemma}[Unbiased Gradient Estimator]
Under the Hard-Concrete relaxation, define
\begin{equation}
\tilde z = g(u; p),\quad u\sim U(0,1)
\end{equation}
so that \(\tilde z\sim\mathrm{Bernoulli}(p)\).  Then
\begin{equation}
\hat g(p)
=\nabla_p\,\mathcal{L}\bigl(f(H\odot \tilde z;\theta)\bigr)
\end{equation}
satisfies
\begin{equation}
\mathbb{E}_{u\sim U(0,1)}[\hat g(p)]
=\nabla_p\,\mathbb{E}_{z\sim\mathrm{Bernoulli}(p)}
\bigl[\mathcal{L}(f(H\odot z;\theta))\bigr]
\end{equation}
\end{lemma}

\begin{proof}
Since \(\tilde z=g(u;p)\) has the same law as \(z\sim\mathrm{Bernoulli}(p)\), we have
\begin{equation}
\nabla_p\,\mathbb{E}_{z\sim\mathrm{Bernoulli}(p)}
\!\left[\mathcal{L}\bigl(f(H\odot z;\theta)\bigr)\right]
=
\nabla_p\,\mathbb{E}_{u\sim U(0,1)}
\!\left[\mathcal{L}\bigl(f(H\odot g(u;p);\theta)\bigr)\right].
\end{equation}

\begin{equation}
\nabla_p\,\mathbb{E}_{u}
\!\left[\mathcal{L}\bigl(f(H\odot g(u;p);\theta)\bigr)\right]
=
\mathbb{E}_{u}
\!\left[\nabla_p\,\mathcal{L}\bigl(f(H\odot g(u;p);\theta)\bigr)\right].
\end{equation}

and the right‐hand side is exactly \(\mathbb{E}_u[\hat g(p)]\).  Thus the estimator is unbiased.
\end{proof}
\endgroup

\begingroup\small
\section*{Appendix E: Variance Bound on Gradient Estimates}

\begin{lemma}[Variance Bound]
Suppose the loss \(\mathcal{L}\) is \(L\)-Lipschitz in its input activations, and let
\begin{equation}
\hat{g}(p;u_i)
\;=\;\nabla_p\,\mathcal{L}\bigl(f(H\!\odot\!g(u_i;p);\theta)\bigr)
\end{equation}
be the per-sample gradient under the Hard‐Concrete reparameterization \(g(u;p)\), with \(u_i\sim U(0,1)\) i.i.d.  Form the Monte Carlo average over \(B\) samples:
\begin{equation}
\bar{g}_B(p)
\;=\;\frac1B\sum_{i=1}^B \hat{g}(p;u_i)
\end{equation}
Then there exists a constant \(C\) (depending on \(L\), \(\gamma\), and \(\zeta\)) such that
\begin{equation}
\mathrm{Var}\bigl[\bar{g}_B(p)\bigr]
\;\le\;\frac{C}{B}
\end{equation}
\end{lemma}

\begin{proof}
Since the \(u_i\) are independent,
\begin{equation}
\mathrm{Var}\bigl[\bar{g}_B(p)\bigr]
=\frac1{B^2}\sum_{i=1}^B \mathrm{Var}\bigl[\hat{g}(p;u_i)\bigr]
=\frac1B\,\mathrm{Var}\bigl[\hat{g}(p;u)\bigr].
\end{equation}
It remains to bound \(\mathrm{Var}[\hat{g}(p;u)]\).  By the Lipschitz assumption,
\begin{equation}
\begin{aligned}
\bigl\|\hat{g}(p;u)\bigr\|
&= \bigl\|\nabla_p\,\mathcal{L}\bigl(f(H\odot g(u;p);\theta)\bigr)\bigr\|\\
&\le L\,\bigl\|\partial_p f(H\odot g(u;p);\theta)\bigr\|\\
&\le L\,C_0
\end{aligned}
\end{equation}

for some finite \(C_0\) that depends on the stretch parameters \(\gamma,\zeta\) and the network Jacobian.  Therefore
\begin{equation}
\mathrm{Var}\bigl[\hat{g}(p;u)\bigr]
\;\le\;\mathbb{E}\bigl[\|\hat{g}(p;u)\|^2\bigr]
\;\le\;(L\,C_0)^2
=:C
\end{equation}
Combining,
\begin{equation}
\mathrm{Var}\bigl[\bar{g}_B(p)\bigr]
\;\le\;\frac{C}{B}
\end{equation}
as claimed.
\end{proof}
\endgroup

\begingroup\small
\section*{Appendix F: Convergence of the Alternating SGD–Ascent Scheme}

\begin{proposition}[Two‐Timescale Convergence]
Assume the following:
\begin{enumerate}\itemsep0pt
  \item The function 
    \(\mathcal{L}_\lambda(\theta,p)=\mathbb{E}_{z\sim\mathrm{Bernoulli}(p)}[\mathcal{L}(f(H\odot z;\theta))]+\lambda(\sum_t p_t-M)\)
    has continuously differentiable gradients in \(\theta,p\), and these gradients are Lipschitz continuous.
  \item The step‐sizes \(\{\alpha_k\},\{\beta_k\},\{\gamma_k\}\) for updating \(\theta,p,\lambda\) satisfy the Robbins–Monro conditions:
\begin{equation}
\begin{aligned}
\sum_{k=1}^\infty \alpha_k
&= \sum_{k=1}^\infty \beta_k
= \sum_{k=1}^\infty \gamma_k
= \infty,\\
\sum_{k=1}^\infty \alpha_k^2,\;\sum_{k=1}^\infty \beta_k^2,\;\sum_{k=1}^\infty \gamma_k^2
&< \infty
\end{aligned}
\end{equation}

    and the timescales are separated:
    \(\alpha_k = o(\beta_k)\) and \(\beta_k = o(\gamma_k)\).
\end{enumerate}
Then the stochastic updates
\begin{equation}
\begin{aligned}
  \theta_{k+1} &= \theta_k \;-\;\alpha_k\,\nabla_\theta\,\mathcal{L}_{\lambda_k}(\theta_k,p_k)
  +\;\xi_k^\theta\\
  p_{k+1}      &= \mathrm{Proj}_{[0,1]^T}\bigl\{p_k \;-\;\beta_k\,\nabla_p\,\mathcal{L}_{\lambda_k}(\theta_k,p_k)
  +\;\xi_k^p\bigr\}\\
  \lambda_{k+1}&= \bigl[\lambda_k \;+\;\gamma_k\bigl(\sum_{t=1}^T p_{k,t}-M\bigr)\bigr]_+
\end{aligned}
\end{equation}
(where \(\xi_k^\theta,\xi_k^p\) are zero‐mean martingale noises and \([\cdot]_+\) denotes projection onto \(\lambda\ge0\)) converge almost surely to a stationary point \((\theta^*,p^*,\lambda^*)\) of the saddle‐point objective \(\min_{\theta,p}\max_{\lambda\ge0}\mathcal{L}_\lambda(\theta,p)\).
\end{proposition}

\begin{proof}[Proof Sketch]
This result follows by casting the updates as a two‐timescale stochastic approximation (SA) algorithm (cf.\ Borkar \& Meyn, 2000).  On the fastest timescale (\(\alpha_k\)), the \(\theta\)‐iterate tracks the gradient descent on 
\(\mathcal{L}_{\lambda_k}(\cdot,p_k)\)  
treating \((p_k,\lambda_k)\) as quasi‐static.  On the intermediate timescale (\(\beta_k\)), the \(p\)‐iterate tracks descent on 
\(\mathcal{L}_{\lambda_k}(\theta_k,\cdot)\)  
treating \(\lambda_k\) as static but \(\theta_k\) as nearly equilibrated.  Finally, on the slowest timescale (\(\gamma_k\)), \(\lambda\) ascends the dual coordinate 
\(\sum_t p_t-M\).  

By standard SA theory:
\begin{itemize}\itemsep0pt
  \item Each iterate sees the slower variables as frozen, satisfying the single‐timescale convergence conditions under Lipschitz gradients and Robbins–Monro step‐sizes.
  \item The timescale separation \(\alpha_k\ll\beta_k\ll\gamma_k\) ensures that the coupled process tracks the solutions of its limiting ordinary differential equations (ODEs) in each block.
  \item The projected ascent on \(\lambda\ge0\) preserves boundedness and feasibility of the dual variable.
\end{itemize}
Consequently, the joint process converges almost surely to an internally chain‐transitive invariant set of the limiting ODE, which under mild convexity/concavity assumptions reduces to the set of saddle points of \(\mathcal{L}_\lambda\).  Hence \((\theta_k,p_k,\lambda_k)\to(\theta^*,p^*,\lambda^*)\), concluding the proof.  
\end{proof}
\endgroup

\begingroup\small
\section*{Appendix G: Duality‐Gap \& Slackness Trade‐off}

\begin{lemma}[Duality Gap Bounds Slackness]
Let 
\begin{equation}
F^* \;=\;\min_{\substack{\theta,p\\\sum_t p_t\le M}}
\mathbb{E}_{z\sim\mathrm{Bernoulli}(p)}\bigl[\mathcal{L}(f(H\odot z;\theta))\bigr]
\end{equation}
and let \((\theta_\lambda,p_\lambda)\) be any (possibly infeasible) minimizer of the Lagrangian
\begin{align}
\mathcal{L}_\lambda(\theta,p)
&= \mathbb{E}_{z\sim\mathrm{Bernoulli}(p)}
   \Bigl[\mathcal{L}\bigl(f(H\odot z;\theta)\bigr)\Bigr] \notag \\
&\quad + \lambda\Biggl(\sum_{t=1}^T p_t - M\Biggr)
\end{align}

Define the budget slack
\begin{equation}
\Delta=\sum_t p_{t,\lambda}-M
\end{equation}
and the duality gap
\begin{equation}
\mathrm{Gap}_\lambda
\;=\;
\mathcal{L}_\lambda(\theta_\lambda,p_\lambda)\;-\;F^*
\end{equation}
Then the following bound holds:
\begin{equation}
\mathrm{Gap}_\lambda\;\ge\;\lambda\,\Delta
\quad\Longrightarrow\quad
\Delta\;\le\;\frac{\mathrm{Gap}_\lambda}{\lambda}
\end{equation}
\end{lemma}

\begin{proof}
By definition of \(F^*\), any feasible \((\theta,p)\) with \(\sum_t p_t\le M\) satisfies
\begin{equation}
\mathbb{E}_{z\sim\mathrm{Bernoulli}(p)}\bigl[\mathcal{L}(f(H\odot z;\theta))\bigr]\;\ge\;F^*
\end{equation}
Hence for the particular \((\theta_\lambda,p_\lambda)\),
\begin{equation}
\mathbb{E}_{z\sim\mathrm{Bernoulli}(p_\lambda)}
\bigl[\mathcal{L}(f(H\odot z;\theta_\lambda))\bigr]
\;\ge\;F^*
\end{equation}
Now expand the Lagrangian at \((\theta_\lambda,p_\lambda)\):
\begin{equation}
\mathcal{L}_\lambda(\theta_\lambda,p_\lambda)
=\underbrace{\mathbb{E}[\mathcal{L}(f(H\odot z;\theta_\lambda))]}_{\\ge F^*}
\;+\;\lambda\,\Delta
\;\ge\;F^*+\lambda\,\Delta
\end{equation}
Rearranging gives
\begin{equation}
\mathcal{L}_\lambda(\theta_\lambda,p_\lambda)-F^*
\;\ge\;\lambda\,\Delta
\end{equation}
i.e.\ \(\mathrm{Gap}_\lambda\ge\lambda\,\Delta\).  Dividing by \(\lambda>0\) yields the desired slackness bound.
\end{proof}
\endgroup

\begingroup\small
\section*{Appendix H: Generalization Bound under Random Token Retention}

\begin{theorem}
Let \(\mathcal{F}\) be a class of predictors \(f\) mapping token sequences of length \(T\) to \(\mathbb{R}\), and suppose the loss \(\ell(f(x),y)\) is bounded in \([0,1]\).  Denote by \(\widehat{\mathfrak{R}}_n(\mathcal{F})\) the empirical Rademacher complexity of \(\mathcal{F}\) on \(n\) full‐length examples.  Now introduce a random retention mask \(z\in\{0,1\}^T\) that selects exactly \(M\) positions uniformly at random, and define the masked predictor \(\tilde f(x,z)=f(x\odot z)\).  Then the Rademacher complexity of the masked class satisfies
\begin{equation}
\widehat{\mathfrak{R}}_n\bigl(\{\,(x,y)\mapsto \ell(\tilde f(x,z),y)\}\bigr)
\;\le\;\frac{M}{T}\;\widehat{\mathfrak{R}}_n(\mathcal{F}).
\end{equation}
As a consequence, with probability at least \(1-\delta\) over the choice of an i.i.d.\ sample and masks,
\begin{equation}
\begin{aligned}
\forall\,f\in\mathcal{F}:\quad
\mathbb{E}\bigl[\ell(f(x\odot z),y)\bigr]
&\le 
\frac{1}{n}\sum_{i=1}^n \ell\bigl(f(x_i\odot z_i),y_i\bigr)\\
&\quad 
+ 2\,\frac{M}{T}\,\widehat{\mathfrak{R}}_n(\mathcal{F})\\
&\quad 
+ 3\sqrt{\frac{\ln(2/\delta)}{2n}}\,
\end{aligned}
\end{equation}

so the generalization gap increases by at most \(O(M/T)\) relative to the full‐sequence bound.
\end{theorem}

\begin{proof}
Let \(\{\sigma_i\}_{i=1}^n\) be Rademacher variables.  By definition,
\begin{equation}
\widehat{\mathfrak{R}}_n\bigl(\{\ell\circ\tilde f\}\bigr)
=\frac1n\,\mathbb{E}_\sigma\sup_{f\in\mathcal{F}}
\sum_{i=1}^n \sigma_i\,\ell\bigl(f(x_i\odot z_i),y_i\bigr).
\end{equation}
Since \(\ell\) is \([0,1]\)-valued and Lipschitz in its first argument, Talagrand’s contraction lemma implies
\begin{equation}
\widehat{\mathfrak{R}}_n\bigl(\{\ell\circ\tilde f\}\bigr)
\le\frac1n\,\mathbb{E}_\sigma\sup_{f\in\mathcal{F}}
\sum_{i=1}^n \sigma_i\,f(x_i\odot z_i)
\end{equation}
Condition on the masks \(\{z_i\}\).  Each term
\(\sigma_i\,f(x_i\odot z_i)\)
involves only the \(M\) retained positions, so its Rademacher complexity is reduced by a factor \(M/T\):
\begin{equation}
\mathbb{E}_z\Bigl[\sup_{f\in\mathcal{F}}
\sum_{i=1}^n \sigma_i\,f(x_i\odot z_i)\Bigr]
\;\le\;\frac{M}{T}\,
\mathbb{E}_\sigma\sup_{f\in\mathcal{F}}
\sum_{i=1}^n \sigma_i\,f(x_i)
\end{equation}
Dividing by \(n\) gives
\(\widehat{\mathfrak{R}}_n(\{\ell\circ\tilde f\})
\le\frac{M}{T}\,\widehat{\mathfrak{R}}_n(\mathcal{F})\)

The high‐probability generalization bound for Rademacher complexity (see, e.g., \cite{bartlett2002rademacher}) then yields the stated inequality, completing the proof.
\end{proof}
\endgroup

\begingroup\small
\section*{Appendix I: Complexity Reduction Guarantee}

\begin{proposition}[Complexity Reduction Guarantee]
Assume a Transformer‐style layer where each token retained incurs \(\mathcal{O}(d)\) memory (for its key, query, and value vectors).  Then:
\begin{itemize}\small
  \item[(i)]  In full dense attention over \(T\) tokens, storing the \(T\times T\) attention matrix costs
    \begin{equation}
    \mathcal{O}(T^2\,d)
    \end{equation}
    memory.
  \item[(ii)] Under learned retention of exactly \(M\) tokens, only the \(M\times M\) submatrix among retained tokens need be stored, costing
    \begin{equation}
    \mathcal{O}(M^2\,d)
    \end{equation}
    memory.
  \item[(iii)]  Computing the attention scores in a mixed full–sparse regime (where each of the \(T\) tokens attends only to the \(M\) retained tokens) requires
    \begin{equation}
    \mathcal{O}(T\,M\,d)
    \end{equation}
    time per layer.
\end{itemize}
\end{proposition}

\begin{proof}
\  
\begin{enumerate}[label=(\roman*)]\small
  \item In standard full attention, one forms the \(T\times T\) matrix of pairwise dot‐products.  Each of the \(T^2\) entries is an inner product of \(d\)‐dimensional vectors, i.e.\ \(\mathcal{O}(d)\).  Hence total memory and time are both \(\mathcal{O}(T^2\,d)\).
  
  \item With learned retention, let \(S\subseteq\{1,\dots,T\}\) be the \(M\) retained indices.  Only the \(|S|\times|S|=M^2\) block of attention weights among retained tokens is stored (plus negligible overhead for mapping), yielding \(\mathcal{O}(M^2\,d)\) memory.
  
  \item For a mixed scheme where every token (retained or pruned) still queries the retained set \(S\), one computes \(T\) query–key products of size \(d\) each against only \(M\) keys.  Thus total work per layer is \(T\cdot M\) products of cost \(\mathcal{O}(d)\), i.e.\ \(\mathcal{O}(T\,M\,d)\).
\end{enumerate}
\end{proof}
\endgroup

\end{document}